\documentclass{article}

\usepackage[main,final]{neurips_2026}

\usepackage[utf8]{inputenc} 
\usepackage[T1]{fontenc}    
\usepackage{hyperref}       
\usepackage{url}            
\usepackage{booktabs}       
\usepackage{amsfonts}       
\usepackage{nicefrac}       
\usepackage{microtype}      
\usepackage{xcolor}         
\usepackage{microtype}
\usepackage{graphicx}
\usepackage{subcaption}
\usepackage{booktabs} 
\usepackage{bm}
\usepackage{multirow}
\usepackage{multicol}
\usepackage{amsmath}
\usepackage{amssymb}
\usepackage{mathtools}
\usepackage{amsthm}
\usepackage{subcaption}
\usepackage{wrapfig}
\usepackage[capitalize,noabbrev]{cleveref}

\theoremstyle{plain}
\newtheorem{theorem}{Theorem}[section]

\theoremstyle{definition}
\newtheorem{definition}[theorem]{Definition}

\theoremstyle{remark}

\newcommand{\method}[0]{LLoPS}
\newcommand{\fullname}[0]{\textbf{L}earnable \textbf{Lo}w-Rank \textbf{P}olynomial \textbf{S}ketch}
\newcommand{\nop}[1]{}

\title{Learnable Low-Rank Polynomial Sketch for Effective Linear Attention}

\author{%
  Zongyue Qin, Haoran Deng, Jason Cong, Yizhou Sun \\
  Department of Computer Science\\
  University of California Los Angeles\\
  CA, USA \\
  \texttt{qinzongyue,denghaoran,cong,yzsun@cs.ucla.edu} \\
}

\begin{document}

\maketitle

\begin{abstract}
    Softmax attention in Transformers suffers from quadratic complexity in sequence length, making it impractical for long-context applications. Linear attention alleviates this issue by replacing the exponential kernel with alternative functions that enable linear-time computation. Among existing linear attention approaches, recent studies have shown that polynomial kernels are particularly effective, as they exhibit sparse and spiky behavior similar to softmax attention, which emphasizes large dot products while suppressing irrelevant interactions. However, the exact computation of high-degree polynomial kernels is infeasible for high-dimensional representations. As a result, prior work relies on approximate polynomial kernels. This introduces a non-negligible approximation error.
In this paper, we show that, from the perspective of polynomial kernel approximation, existing linear attention methods are still suboptimal. We propose \fullname{} (\method{}), a principled and flexible framework for approximating polynomial kernels with linear attention. Our method learns a low-rank polynomial sketch that provably achieves a strictly smaller approximation error than existing approaches.
Experiment results show that \method{} achieves the highest performance across extensive benchmarks, comparing with various linear attention baselines. 
\end{abstract}

\section{Introduction}

Transformers have become the dominant architecture for sequence modeling across natural language processing, vision, and multimodal tasks. Despite their empirical success, the quadratic time and memory complexity of the attention mechanism with respect to sequence length severely limits scalability to long-context settings. This bottleneck is particularly problematic in emerging applications such as long-document understanding, code modeling, and retrieval-augmented generation, where context lengths can easily reach tens or hundreds of thousands of tokens.

Linear attention methods~\citep{xiong2021nystromformer,team2025kimi,schlag2021deltanet} address this challenge by retaining the core attention formulation while reducing its complexity from quadratic to linear in sequence length. By rewriting softmax attention as a kernelized dot-product and exploiting associativity, linear attention enables streaming computation with constant memory. Recent large-scale deployments~\citep{team2025kimi} demonstrate that linear attention can attain a strong empirical performance while supporting extremely long contexts, thus making it a promising direction for next-generation foundation models.

Existing work on linear attention can be broadly categorized into two complementary directions. The first focuses on improving the state update rule, designing better ways to accumulate and normalize the running key-value statistics over time~\citep{team2025kimi,yang2024gated}. The second direction, which this paper focuses on, aims to improve the feature map of query and key vectors so that the induced kernel more faithfully approximates softmax attention~\citep{arora2024based,zhang2024hedgehog}.
Prior studies on feature map design have observed that high-order polynomial kernels can approximate the softmax kernel surprisingly well~\citep{kacham2023polysketchformer,arora2024based,zhang2024hedgehog}. Intuitively, both exponential and polynomial kernels can exhibit sparse, spiky behavior that emphasizes large dot-products while suppressing irrelevant interactions. However, computing exact polynomial kernels is infeasible for high degree kernel and high-dimensional vectors, as the computation complexity grows exponentially with the polynomial degree.

To address this challenge, PolySketchFormer~\citep{kacham2023polysketchformer} exploits a data-independent sketching technique to approximate exact polynomial kernels. More recently, BASED~\citep{arora2024based} realizes state-of-the-art performance in language modeling by combining linear attention and sliding window attention. In their linear attention formulation, queries and keys are first compressed into a low-dimensional space and then the exact polynomial kernel is computed in that space.

In this paper, we show that from the perspective of polynomial kernel approximation, the previous methods that approximate polynomial kernel are suboptimal. Our key insight is that, when the output feature dimension is fixed, there exists a theoretically optimal sketch that minimizes the approximation error to the target polynomial kernel. We prove that a linear projection for the $p$-th order outer products of input vectors can achieve this optimal approximation, whereas existing approaches such as BASED~\citep{arora2024based} and PolySketchFormer~\citep{kacham2023polysketchformer} can be viewed as constrained or data-independent approximations to this optimal mapping.
Building on this observation, we further propose a structured decomposition of the optimal linear mapping using low-rank tensor rank decompositions (also known as CP decompositions) and column-wise Kronecker products, which significantly reduces computational costs. We provide theoretical analysis showing that this structured approximation is strictly tighter than prior feature-map constructions under the same dimensional budget.

Our contributions are summarized as follows: 
(1) We provide a theoretical characterization of the optimal sketching strategy for approximating polynomial kernels in the context of linear attention. 
(2) Motivated by this analysis, we propose a novel \fullname{} (\method{}), a linear attention kernel that accurately approximates polynomial attention while remaining computationally efficient. 
(3) We further derive approximation error analysis for \method{} and show that several existing linear attention kernels can be interpreted as special cases within our unified framework. 
(4) We demonstrate through experiments that \method{} consistently improves the effectiveness of linear attention transformers across a broad range of benchmarks. 

\section{Preliminaries}

\subsection{Softmax-based Full Attention}
Given a sequence of queries $\bm{Q} \in \mathbb{R}^{n \times d}$, keys $\bm{K} \in \mathbb{R}^{n \times d}$, and values $\bm{V} \in \mathbb{R}^{n \times d_v}$, standard Transformers compute attention as
$
    \mathrm{Attn}(\bm{Q,K,V})
    =
    \mathrm{softmax}\!\left(\frac{\bm{QK}^\top}{\sqrt{d}}\right)\bm{V} ,
$
where the softmax is applied row-wise. Because this formulation has $\mathcal{O}(n^2)$ time and memory complexity due to the explicit computation of the attention matrix $\bm{QK}^\top$, it is prohibitively expensive for long input sequences.

\subsection{Linear Attention}\label{sec:prelim}
Previous studies have observed that softmax attention can be interpreted as kernelized dot-product attention with an exponential kernel. By approximating the exponential function with a feature map $\phi:\mathbb{R}^{d} \rightarrow \mathbb{R}^m$ and exploiting associativity, attention can be rewritten as
$
    \mathrm{Attn}(\bm{Q,K,V})=\left(\phi(\bm{Q})\phi(\bm{K})^\top\right) \bm{V}
    =
    \phi(\bm{Q})\big(\phi(\bm{K})^\top \bm{V}\big),
$
where $\phi(\bm{Q}), \phi(\bm{K}) \in \mathbb{R}^{n \times m}$ are mapped features with dimension $m$ and $\phi(\bm{Q})\phi(\bm{K})^\top$ is the inner product in the new feature space that approximates $\mathrm{softmax}\!\left(\frac{\bm{QK}^\top}{\sqrt{d}}\right)$.\footnote{We use $\phi(\bm{Q})$ to denote applying $\phi(\cdot)$ to each row vector in $\bm{Q}$. Similar for $\phi(\bm{K})$ and later usages.}
Crucially, the overall complexity is reduced from quadratic to linear in the sequence length $n$. This enables streaming computation and a constant memory state, which are key advantages of linear attention.

Recent work reveals that polynomial kernels provide a strong approximation to the exponential kernel underlying softmax attention. In particular, a degree-$p$ polynomial kernel, denoted as $\kappa(\cdot, \cdot)$, takes the form
$
    \kappa(\bm{q,k}) = (\bm{q}^\top \bm{k})^p ,
$
where $\bm{q,k}\in\mathbb{R}^d$ are query and key vectors.

However, computing the exact polynomial feature map required by linear attention involves expanding all degree-$p$ monomials, i.e.,
$
    \psi_p(\bm{x}) = \mathrm{vec}\!\left(\bm{x}^{\otimes p}\right),
$
where $\otimes p$ is $p$-th order outer product, whose dimensionality scales as $\mathcal{O}(d^p)$. This makes exact polynomial kernels impractical for real-world models.
BASED addresses this issue by first compressing the query and key vectors into a low-dimensional space using a linear projection,
\begin{equation}
    \tilde{\bm{q}} = \bm{W}_q \bm{q}, \quad \tilde{\bm{k}} = \bm{W}_k \bm{k}, \qquad \bm{W}_q,\bm{W}_k \in \mathbb{R}^{r \times d}, \; r \ll d ,
\end{equation}
and then computing exact polynomial features in the compressed space. In practice, BASED utilizes a second-order polynomial kernel
$
    \kappa(\bm{q,k}) = \psi_2(\bm{q})^{\top} \psi_2(\bm{k}) ,
$
which corresponds to the quadratic term in the Taylor expansion of the exponential kernel.\footnote{BASED proposes to use the Taylor expansion $1 + x + x^2/2$. For simplicity, we omit the constant and linear terms, as they can be handled by concatenation and do not affect the core analysis.}
By combining this polynomial linear attention with sliding-window attention to capture local interactions, BASED achieves state-of-the-art language modeling performance while maintaining favorable computational efficiency.

\subsection{Related Work}

We broadly categorize prior studies on linear attention into two directions: (i) the design of feature maps for query and key vectors, and (ii) the design of state update rules.

\paragraph{Feature Map Design.}
This line of work focuses on constructing efficient and expressive feature maps $\phi(\cdot)$ such that the softmax kernel $\exp(\bm{q}^\top\bm{k})$ can be approximated by an inner product $\phi(\bm{q})^\top\phi(\bm{k})$~\citep{katharopoulos2020transformers,choromanski2020performer,qin2022cosformer,keles2023computational,meng2025polaformer}. Among the many design choices, several recent studies have identified polynomial kernels as particularly effective approximations, as they share important inductive biases with softmax attention, such as sparsity and sharp selectivity~\citep{han2023flatten,zhang2024hedgehog,kacham2023polysketchformer,arora2024based}. As discussed by \citep{arora2024based}, many existing feature maps are substantially less effective than polynomial-kernel-based approaches when evaluated in realistic settings. Moreover, for current methods that explicitly approximate polynomial kernels~\citep{han2023flatten,kacham2023polysketchformer,arora2024based}, their approximation accuracy is still limited. Our work finds that these approximations can be further improved through a more principled and flexible feature-map construction. Overall, there remains significant room for boosting the effectiveness of linear attention feature maps.

Additionally, despite the extensive exploration of feature map design, many prior works do not assess performance on large-scale, real-world language modeling tasks. In some cases, linear attention variants are reported to outperform softmax attention, which contrasts with empirical observations in large language models where linear attention usually has sub-optimal effectiveness. 

\paragraph{State Update Rules.}
Another line of work focuses on designing improved state update rules for linear attention that aim to stabilize training and better capture long-range dependencies~\citep{schlag2021deltanet,yang2024gated,team2025kimi}. These methods often resemble state space models~\citep{gu2024mamba,peng2023rwkv} in that they maintain and update a recurrent state, rather than strictly adhering to the standard attention formulation.

Importantly, these approaches are largely orthogonal to feature map design: their update rules can be combined with different feature maps by replacing $\bm{q}$ and $\bm{k}$ with $\phi(\bm{q})$ and $\phi(\bm{k})$. Since our work focuses exclusively on improving feature maps for polynomial kernel approximation, we do not directly compare with methods that primarily modify the state update mechanism.

\paragraph{Polynomial Sketches for Neural Networks.}
Polynomial sketches have also been incorporated into neural network modules in prior work. For example, PolySketchFormer~\citep{kacham2023polysketchformer} uses polynomial sketches to approximate attention, and EUGens~\citep{kim2026eugens} applies polynomial sketches within feedforward layers. While these works demonstrate the practical usefulness of polynomial sketching, their theoretical analyses are based on data-independent sketches, where the sketching map is fixed independently of the data distribution. This setting differs fundamentally from ours: \method{} learns a data-dependent sketch tailored to the target representations, which can yield substantially better approximation under a fixed sketch dimension. To the best of our knowledge, our work is the first to provide theoretical analysis and approximation guarantees for such data-dependent polynomial sketches.
\section{Methodology}

Polynomial kernel is considered to be a good candidate to replace softmax kernel~\citep{kacham2023polysketchformer,zhang2024hedgehog,arora2024based}.
A key reason for the effectiveness of softmax attention is its sparse and spiky attention distribution: for a given query, only a small subset of keys receives large attention weights, while the majority of keys are assigned values close to zero. This behavior arises from the exponential function, which sharply amplifies differences in inner products, effectively separating keys that are well aligned with the query from those that are not.

\begin{figure}
    \centering
    \includegraphics[width=0.46\textwidth]{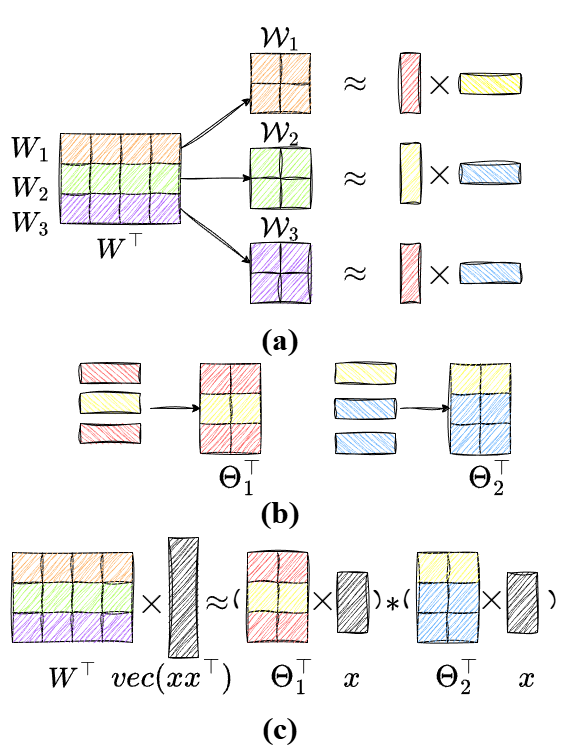}
    \caption{Illustration of \method{} with an example where $p=2,d=2,m=3$. (a) Given the optimal linear projection matrix $\bm{W}$, each column of $\bm{W}$ can be written as a 2 by 2 matrix, and be approximated by its rank-1 CP decomposition; (b) The rank-1 CP decomposition can be stacked into two weight matrices $\bm{\Theta_1},\bm{\Theta_2}$; (c) In this way, the linear projection $\bm{W}^\top\text{vec}(\bm{xx}^\top)$ can be approximated as $(\bm{\Theta_1}^\top\bm{x})*(\bm{\Theta_2}^\top\bm{x})$ where $*$ is an element-wise product.}
    \label{fig:LPS}
    \vspace{-10pt} 
\end{figure}

Polynomial kernels exhibit a similar spikiness property, as higher-order powers increasingly emphasize large inner products while suppressing smaller ones. Besides, polynomial functions are often used to approximate exponential function (e.g., Taylor expansion)~\citep{zhang2024hedgehog,arora2024based}. Thus, a polynomial kernel is a natural and effective alternative to the exponential kernel in the context of linear attention. However, as discussed in Sec \ref{sec:prelim}, an exact computation of polynomial kernels becomes impractical for high degrees or high-dimensional vectors due to the exponential growth of dimensionality.

Therefore, our work focuses on designing feature maps for queries and keys separately, denoted by $\phi_Q(\bm{Q}), \phi_K(\bm{K}):\mathbb{R}^{n\times d}\rightarrow\mathbb{R}^{n\times m}$, that accurately approximate polynomial kernels while remaining computationally efficient. Specifically, we aim to minimize the following approximation error:
\begin{equation}
\left\| \phi_Q(\bm{Q})\phi_K(\bm{K})^\top - (\bm{Q}\bm{K}^\top)^{* p} \right\|_F^2, \label{eq:obj}
\end{equation}
where $* p$ is the element-wise $p$-th power and $\|\cdot\|_f$ is Frobenius norm.

To achieve this goal, we introduce \fullname{} (\method{}). We first demonstrate that, when the output dimension $m$ is fixed, complicated feature maps are unnecessary to achieve the best compression. Instead,
linear projections of $\bm{q}^{\otimes p}$ and $\bm{k}^{\otimes p}$ are sufficient to attain the theoretically optimal approximation to the target polynomial kernel. Building on this, we further decompose the optimal linear projection and approximate it using a low-rank CP (CANDECOMP/PARAFAC) decomposition~\citep{papalexakis2016automatic}, which is also known as tensor rank decomposition, that enables a practical computation with linear attention while preserving the approximation quality. Finally, we provide a detailed error analysis of the proposed feature maps and elucidate their connections to existing linear attention kernels.

\subsection{Optimal Sketch of Polynomial Kernel}

We first consider the optimal compression of polynomial kernel when the output dimension $m$ is fixed. For clarity, let $\psi_p:\mathbb{R}^d\rightarrow \mathbb{R}^{d^p}$, where $\psi_p(\bm{x})=\text{vec}(\bm{x}^{\otimes p})$; let $\bm{A}=(\bm{Q}\bm{K}^\top)^{* p}=\psi_p(\bm{Q})\psi_p(\bm{K})^\top$\footnote{Function $\psi_p$ is applied to each row of the matrices.}; and 
let $\bm{U\Sigma V}^\top$ be the SVD decomposition of $\bm{A}$ where $\bm{\Sigma}=diag(\sigma_1,\cdots,\sigma_n)$.

Since $\phi_Q(\bm{Q})$ and $\phi_K(\bm{K})$ are both an $n\times m$ matrix ($n\gg m$), the rank of $\phi_Q(\bm{Q})\phi_K(\bm{K})^\top$ is at most $m$. By the property of SVD decomposition, the minimum possible loss of Eq. \ref{eq:obj} is 
$
 \sum_{i=m+1}^{n}\sigma_{i}^2   
$ where $\sigma_i$ is the $i$-th singular value.
It is achieved when
\begin{equation}
    \phi_Q(\bm{Q})\phi_K(\bm{K})^\top=\bm{U\Sigma_m V}^\top
    \label{eq:optimal}
\end{equation}
where
$   \bm{\Sigma}_m=diag(\sigma_1,\cdots,\sigma_m,0,\cdots,0)
$.

The following theorem shows how to achieve the minimum approximation error.

\begin{theorem}[Optimal Polynomial Sketch]
    \label{th:optimal} Let $\psi_p(\bm{x}):\mathbb{R}^d\rightarrow \mathbb{R}^{d^p}=\text{vec}(\bm{x}^{\otimes p})$.
    There exists $\bm{W}^*_Q,\bm{W}^*_K\in \mathbb{R}^{d^p\times m}$ such that $\phi_Q^*(\bm{Q})\phi_K^*(\bm{K})^\top=\bm{U\Sigma_m V}^\top$ where 
    \begin{equation}
            \phi_Q^*(\bm{Q})=\psi_p(\bm{Q})\bm{W}^*_Q\quad 
            \phi_K^*(\bm{K})=\psi_p(\bm{K})\bm{W}^*_K
    \end{equation}
\end{theorem}
The proof is illustrated in Appendix \ref{app:proof_optimal}. 

Theorem \ref{th:optimal} suggests that we could compress $\psi_p(\bm{Q}),\psi_p(\bm{K})$ into a $m$-dimensional space via linear mapping, while
achieving the minimal approximation error stated in Eq \ref{eq:optimal}.

\subsection{\fullname{}}

Although Theorem \ref{th:optimal} illustrates how to achieve the optimal approximation,
since $\bm{W}_Q^*$ and $\bm{W}_K^*$ have $d^p$ rows, the linear projection is still too computationally expensive. To address this problem, we propose to decompose and approximate the linear projection matrices utilizing their low-rank CP decompositions (also known as tensor rank decomposition~\citep{papalexakis2016automatic}). Figure \ref{fig:LPS} depicts an example of our decomposed sketch for a toy example where $p=2,d=2,m=3$.

Let $\bm{W}$ denote $\bm{W}^*_Q$ or $\bm{W}^*_K$, and let $\bm{W}_i\in\mathbb{R}^{d^p}$ ($1\le i\le m$) denote the $i$-th column of $\bm{W}$. We can write $\bm{W}_i$
as a $(p)$-way tensor $\mathcal{\bm{W}}_i \in \mathbb{R}^{d\times \cdots \times d}$. According to CP decomposition, $\mathcal{\bm{W}}_i$ can be represented, with a sufficiently large $R_i$, as a linear combination of $R_i$ rank-1 tensors:
\begin{equation}
    \mathcal{\bm{W}}_i=\sum_{r=1}^{R_i} \lambda_r^{(i)} \bm{w}_{1,r}^{(i)}\otimes \bm{w}_{2,r}^{(i)}\otimes \cdots \otimes \bm{w}_{p,r}^{(i)}
    \label{eq:cp_decomposition}
\end{equation}
where $\lambda_r^{(i)}\in\mathbb{R}$, $\bm{w}_{j,r}^{(i)}\in\mathbb{R}^d$ ($1\le j \le p$), and $\otimes$ denotes outer product. 

Inspired by the CP decomposition in Eq.~\ref{eq:cp_decomposition}, the key idea of \method{} is as follows. 
Each row of $\bm{W}$ is expressed as a sum of rank-one outer products of learnable vectors, as illustrated in Figure~\ref{fig:LPS}(a). 
These vectors are trained implicitly to approximate the low-rank CP decomposition of each row vector $\bm{W}_i$. 
We then organize the learned factors into a collection of projection matrices, as shown in Figure~\ref{fig:LPS}(b), which enables an efficient computation of the feature map, avoiding explicitly constructing the high-dimensional polynomial features $\psi_p(\bm{Q})$ and $\psi_p(\bm{K})$ through a sequence of lightweight matrix multiplications and element-wise operations instead, as depicted in Figure~\ref{fig:LPS}(c).

Specifically, $\mathcal{W}_i$ can be approximated with its rank-$t$ CP decomposition, i.e.,
\begin{equation}
    \mathcal{\bm{W}}_i\approx \sum_{r=1}^{t} \lambda_r^{(i)} \bm{w}_{1,r}^{(i)}\otimes \bm{w}_{2,r}^{(i)}\otimes \cdots \otimes \bm{w}_{p,r}^{(i)}\triangleq \hat{\mathcal{\bm{W}}}_i
    \label{eq:define_approx}
\end{equation}

By reshaping the tensor back to vector,  we have an approximate column vector $\hat{\bm{W}}_i=\text{vec}(\hat{\mathcal{W}}_i)$ and
thus the approximate linear projector $
    \hat{\bm{W}}=(\hat{\bm{W}_1},\cdots,\hat{\bm{W}_m})\in\mathbb{R}^{d^p\times m}
$. Next, we demonstrate how to 
efficiently compute $\psi_p(\bm{Q})\hat{\bm{W}}$ and $\psi_p(\bm{K})\hat{\bm{W}}$ without explicitly computing $\psi_p(\bm{Q})$, $\psi_p(\bm{K})$, and $\hat{\bm{W}}$, which are with extremely high dimensions. 

Notice that the inner product between the outer product of vectors $\{\bm{w}_{1,r},\ldots,\bm{w}_{p,r}\}$ and $\text{vec}(\bm{x}^{\otimes p})$
equals the product of $\{\langle\bm{w}_{j,r},\bm{x}\rangle: 1\le j\le p\}$, i.e.,
\begin{equation}
    \langle \text{vec}(\bm{w}_{1,r}\otimes\cdots\otimes \bm{w}_{p,r}),\text{vec}(\bm{x}^{\otimes p}) \rangle=\prod_{j=1}^p \langle \bm{w}_{j,r}, \bm{x}\rangle
\end{equation}
So we have
$
    \langle\hat{\bm{W}}_i,\text{vec}(\bm{x}^{\otimes p})\rangle
    =\sum_{r=1}^{t}\lambda_r^{(i)}\prod_{j=1}^{p}\langle \bm{w}_{j,r}^{(i)},\bm{x}\rangle
$.

Thus
\begin{equation}
\begin{aligned}
    &\psi_p(\bm{Q})\hat{\bm{W}}\\=&(\psi_p(\bm{Q})\hat{\bm{W}}_1,\cdots,\psi_p(\bm{Q})\hat{\bm{W}}_m)\\
    =&\big(\sum_{r=1}^{t}\lambda_{r}^{(1)}\prod_{j=1}^{p}\bm{Q} \bm{w}_{j,r}^{(1)},\cdots, \sum_{r=1}^{t}\lambda_{r}^{(m)}\prod_{j=1}^{p}\bm{Q} \bm{w}_{j,r}^{(m)}\big)\\
    =&\sum_{r=1}^{t}\big( \lambda_{r}^{(1)}\prod_{j=1}^{p}\bm{Q} \bm{w}_{j,r}^{(1)},\cdots, \lambda_{r}^{(m)}\prod_{j=1}^{p}\bm{Q} \bm{w}_{j,r}^{(m)}\big) 
\end{aligned}
\label{eq:decompose}
\end{equation}

With Eq \ref{eq:decompose}, we have the following theorem to approximate the optimal projection matrix $\bm{W}$ with a set of matrices $\bm{\Theta}$. 
\begin{theorem}
    Let $\bm{\Theta}_{j,r}\in\mathbb{R}^{d\times m}$ be the stack of $\bm{w}_{j,r}^{(i)}$ for $i=1,\ldots,m$, i.e.,
\begin{equation}
    \bm{\Theta}_{j,r} = (\lambda_r^{(1)} \bm{w}_{j,r}^{(1)}, \cdots, \lambda_r^{(m)}\bm{w}_{j,r}^{(m)})
\end{equation}
where $1\le j\le p$, $1\le r \le t$. We have
\begin{equation}
\begin{aligned}
    &\psi_p(\bm{Q})\hat{\bm{W}}=\sum_{r=1}^t\prod_{j=1}^p \bm{Q} \bm{\Theta_{j,r}},\quad \\
   &\hat{\bm{W}}=\sum_{r=1}^t \big(\bm{\Theta_{1,r}}\odot_{col}\bm{\Theta_{2,r}}\odot_{col}\cdots\odot_{col}\bm{\Theta_{p,r}}\big)
\end{aligned}
    \label{eq:decomposition} 
\end{equation}
where $\odot_{col}$ is the column-wise Kronecker product.
Similarly, $\psi_p(\bm{K})\hat{\bm{W}} = \sum_{r=1}^t\prod_{j=1}^p \bm{K} \bm{\Theta_{j,r}}$. 
\label{th:decompose}
\end{theorem}

More importantly, this formulation avoids explicitly constructing the exponentially large polynomial feature map $\psi_p(\bm{Q})$: we first compute the projected matrices $\bm{Q}\bm{\Theta}_{j,r}$ for all $j=1,\dots,p$ and $r=1,\dots,t$ using standard matrix multiplications. Then, for each rank component $r$, we combine the $p$ projected terms via element-wise multiplication, corresponding to the product over $j$. Finally, we sum the resulting features across all $t$ components. Overall, the total computation cost is reduced from $O(nmd^p)$ to $O(ndmtp)$.

In practice, we find that the setting $t=1$ already attains an accurate approximation of the polynomial kernel. Setting $t>1$ will lead to too many learnable parameters. Therefore, we fix $t=1$ and our proposed feature maps is formulated as below:
\begin{definition}[\fullname{}]
\begin{equation}
    \begin{aligned}
        \phi_Q(\bm{Q})=\prod_{j=1}^p \bm{Q} \bm{\Theta_{j}}^{(Q)},\quad
        \phi_K(\bm{K})=\prod_{j=1}^p \bm{K} \bm{\Theta_{j}}^{(K)}
    \end{aligned}
\end{equation}   
where $\bm{\Theta_{j}}^{(Q)}$ and $\bm{\Theta_{j}}^{(Q)}$ ($1\le j \le p$) are learnable.
\end{definition}
The computation complexity is $O(ndmp)$, and the number of trainable parameters are $O(dmp)$. During training, $\bm{\Theta}$ are trained implicitly to approximate the low-rank CP decomposition of the optimal projecion matrix $\bm{W}_Q^*$ and $\bm{W}_K^*$.

The following theorem analyze the error of our proposed learnable polynomial sketch.
\begin{theorem}[\method{} Error]\label{th:error}
     Let $\bm{W}^*\in\mathbb{R}^{d^p\times m}$ be the optimal matrix suggested in Theorem \ref{th:optimal}. Then 
\begin{equation}
    \min_{\bm{\Theta}_{},\ldots,\bm{\Theta}_{p}}\|\bm{W}^*-\bm{\Theta}_{1}\odot\ldots \odot \bm{\Theta}_{p}\|_F^2=\|\bm{W}^{*}\|_F^2-\sum_{i=1}^{m}\tau_i^2
\end{equation}
where 
\begin{equation}
    \tau_i=\max_{\|\bm{u}_1\|=\ldots=\|\bm{u}_p\|=1}\langle \bm{W}^*[:,i],\bm{u}_1\otimes\ldots \otimes \bm{u}_p\rangle
\end{equation}
is the tensor spectral norm (i.e., the largest tensor singular value) of $\mathcal{W}_i$, and $\mathcal{W}_i\in\mathbb{R}^{d\times\ldots\times d}$ is the order $p$ tensor by reshaping the $i$th column of $\bm{W}^*\in\mathbb{R}^{d^p}$.
\end{theorem}
The proof is given in Appendix \ref{app:proof-error}.

Some previous studies found that making the attention score of linear attention to be non-negative is helpful for training stability~\citep{kacham2023polysketchformer,han2023flatten}. If non-negativity is required, we could adopt the approach used in \citep{kacham2023polysketchformer} to apply another square to the feature map, i.e., $\phi_Q(\bm{Q})=\big(\prod_{j=1}^p \bm{Q} \bm{\Theta_{j}}^{(Q)}\big)^{*2}$, and similarly for $\phi_K(\bm{K})$.

\subsection{Relationship with Existing Feature Maps}

In this subsection, we discuss the relationship between our proposed feature map and several existing approaches for approximating polynomial kernels in linear attention.

\paragraph{Element-wise Polynomial Feature Maps} Some earlier linear attention methods~\citep{han2023flatten} adopt element-wise polynomial feature maps of the form
$
    \phi(\bm{x}) = (\bm{\Theta}\bm{x})^{* p}
$
where $* p$ denotes the element-wise $p$-th power. This construction can be viewed as a special case of our formulation in which all projection matrices are shared, i.e., $\bm{\Theta}_1 = \cdots = \bm{\Theta}_p$. Under this constraint, our kernel reduces exactly to an element-wise polynomial feature map.

\paragraph{BASED} This method~\citep{arora2024based} approximates the second-order polynomial kernel by first projecting queries and keys into a low-dimensional space via a linear mapping, and then applying the quadratic feature map $\psi_2(\cdot)$. Specifically, for $\bm{W}\in\mathbb{R}^{m'\times d}$ and $\bm{x}\in\mathbb{R}^d$, BASED computes
\begin{equation}
    \psi_2(\bm{W}\bm{x})
    =
    \mathrm{vec}(\bm{W}\bm{x}\otimes \bm{W}\bm{x})
    =
    (\bm{W}\otimes_{\mathrm{kron}} \bm{W})\,\mathrm{vec}(\bm{x}\bm{x}^\top),
\end{equation}
where $\otimes_{\mathrm{kron}}$ is the full Kronecker product.

This construction can also be recovered as a special case of our kernel. In particular, let
\[
\bm{\Theta}_1^\top = \bm{W}\odot_{\mathrm{col}}\bm{1},
\qquad
\bm{\Theta}_2^\top = \bm{1}\odot_{\mathrm{col}}\bm{W},
\]
where $\bm{1}\in\mathbb{R}^{m'\times d}$ is an all-ones matrix. Under this choice, we have
$
\bm{\Theta}_1 \odot_{\mathrm{col}} \bm{\Theta}_2
=
(\bm{W}\otimes_{\mathrm{kron}}\bm{W})^\top.
$
Therefore, with these specific parameter settings, our feature map exactly recovers the BASED kernel.

\paragraph{PolySketchFormer}
This approach~\citep{kacham2023polysketchformer} approximates polynomial kernels using randomized sketching techniques. Its feature map can be expressed as multiplication by a randomized sketching matrix $\bm{R}\in\mathbb{R}^{d^p\times m}$ satisfying the $(\epsilon,p)$-AMM property~\citep{woodruff2014sketching}. Following prior constructions~\citep{ahle2020oblivious}, $\bm{R}$ can be formulated as a column-wise Kronecker product of $p$ sparse randomized matrices, permitting efficient approximation of high-order polynomial kernels.

From our perspective, PolySketchFormer corresponds to a special case of our framework in which the projection matrices $\{\bm{\Theta}_i\}_{i=1}^p$ are fixed, randomized, and sparse, rather than learned from data.

Taken together, these observations reveal that several existing polynomial-kernel-based linear attention methods can be interpreted as special or constrained instances of our proposed feature map. By allowing more general and learnable projections, our formulation strictly generalizes these prior approaches. As a consequence, for a fixed output dimension, our method can achieve an approximation error that is no worse than these existing methods, and in practice often substantially better, as shown in the next section. This is equivalent to enforcing a constraint where half of the learnable matrices share parameters with the other half.

\section{Experiments}

In this section, we report the main experiment results comparing our method with state-of-the-art linear attention methods. Additional experiments and ablation studies are reported in Appendix \ref{app:experiment}.

\subsection{Zero-Shot NLP Benchmarks}

Following prior work on alternative sequence modeling architectures~\citep{gu2024mamba,arora2024based}, we evaluate our method on a suite of zero-shot NLP benchmarks. We compare against the standard softmax Transformer, Mamba, and BASED, using the publicly released 360M and 1B parameter checkpoints from~\citet{arora2024based}. Given that these architectures require extensive pre-training from scratch, additional linear attention baselines were excluded from this evaluation.

Since the BASED attention kernel can be viewed as a special case of our proposed kernel, we initialize our model from the BASED checkpoint and continue training it to obtain our final model to reduce training time (BASED+\method{}). For \method{}-360M, the polynomial degree $p=4$; for \method{}-1B $p=2$. Importantly, the original baselines in~\citet{arora2024based} were pretrained on the Pile dataset~\citep{gao2020pile}, which is no longer publicly available. Therefore, we perform continued training utilizing the SlimPajama-6B dataset~\citep{cerebras2023slimpajama}\footnote{\url{https://huggingface.co/datasets/DKYoon/SlimPajama-6B}}, which is 50 times smaller in size compared to Pile. For a fair comparison under this updated data setting, the BASED models are also post-trained with SlimPajama-6B for the same number of iterations. This allows us to isolate the benefits of our proposed kernel from differences arising purely from additional training data or compute. The evaluated benchmarks are implemented in \texttt{lm\_eval}~\citep{eval-harness}, including: WinoGrande~\citep{sakaguchi2021winogrande}, PIQA~\citep{bisk2020piqa}, LAMBADA~\citep{paperno2016lambada}, HellaSwag~\citep{zellers2019hellaswag}, ARC-challenge~\citep{clark2018arc}, SQUAD~\citep{rajpurkar2016squad}, SWDE~\citep{lockard2019swde,arora2023fda}, FDA~\citep{arora2023fda}. The tasks include information extraction and commonsense reasoning QA.

The results are reported in Table~\ref{tab:nlp_results}. Notably, for 360M models, we observe that after continued training on SlimPajama-6B, both BASED and our method exhibit a substantial accuracy drop on the FDA benchmark. We hypothesize that this degradation arises from a mismatch between the SlimPajama-6B and the FDA task distribution, which may induce catastrophic forgetting of task-relevant knowledge. Since this phenomenon occurs for both BASED and our method, it is more likely due to the choice of training data rather than a limitation of our proposed kernel.
Despite this issue, our method produces a higher average accuracy than BASED. These results further support the effectiveness of our approach in improving the performance of linear attention models across diverse zero-shot NLP benchmarks.

\begin{table*}[t]
\centering
\scriptsize
\caption{Zero-shot accuracy (\%) across diverse NLP benchmarks for pretrained models at two scales. 
Bold numbers indicate the best performance among linear-complexity methods within each model size. BASED+CT is BASED under the same continual training setting as \method{}.}
\label{tab:nlp_results}
\begin{tabular}{lcccccccccc}
\toprule
\multicolumn{11}{c}{\textbf{360M Models}} \\
\midrule
Method & Wino & PIQA & LAM & Hella & ARC-E & ARC-C & FDA & SWDE & SQuAD & Avg \\
\midrule
Softmax Attn & 51.78 & 63.71 & 29.32 & 33.60 & 42.85 & 24.32 & 57.89 & 56.08 & 27.88 & 43.05 \\
\midrule
Mamba       & 50.36 & 63.76 & 27.89 & 33.88 & 42.42 & 24.66 &  5.90 & 17.37 & 24.83 & 32.34 \\
BASED       & 51.30 & 63.55 & 26.64 & 33.23 & 41.88 & 24.23 & \textbf{17.42} & 27.18 & 27.41 & 34.76 \\
BASED+CT& 50.91 & 65.40 & 27.73 & 36.85 & 43.10 & \textbf{25.51} &  \textbf{8.35} & 30.87 & 26.27 & 35.00 \\
BASED+\textbf{\method{}}
            & \textbf{53.04} & \textbf{66.70} & \textbf{29.23} & \textbf{37.07} 
            & \textbf{43.94} & \textbf{25.51} &  7.53 & \textbf{31.41} 
            & \textbf{31.94} & \textbf{36.26} \\
\midrule\midrule

\multicolumn{11}{c}{\textbf{1B Models}} \\
\midrule
Method & Wino & PIQA & LAM & Hella & ARC-E & ARC-C & FDA & SWDE & SQuAD & Avg \\
\midrule
Softmax Attn & 52.49 & 67.30 & 38.07 & 39.12 & 51.05 & 22.10 & 68.78 & 68.68 & 35.79 & 49.26 \\
\midrule
Mamba       & 51.30 & 66.81 & 38.87 & 39.34 & 46.80 & 26.11 & 11.07 & 28.17 & 29.46 & 37.55 \\
BASED       & 51.62 & 66.70 & 37.86 & 39.10 & 45.54 & 25.60 & 24.05 & 42.84 & 34.58 & 40.88 \\
BASED+CT& 51.22 & 68.44 & 36.43 & 41.59 & 47.01 & \textbf{27.30} & 24.41 & 46.80 & 36.23 & 42.16 \\
BASED+\textbf{\method{}}
            & \textbf{51.62} & \textbf{68.72} & \textbf{38.11} & \textbf{42.69} 
            & \textbf{47.31} & 27.13 & \textbf{26.32} & \textbf{47.97} 
            & \textbf{36.36} & \textbf{42.91} \\
\bottomrule
\end{tabular}
\end{table*}

\subsection{Long Range Arena Benchmark}

Following prior work~\citep{zhang2024hedgehog,meng2025polaformer,han2023flatten}, we evaluate \method{} on the Long Range Arena (LRA) benchmark~\citep{tay2020lra}. Instead of the original train-from-scratch protocol, we adopt the pretraining-and-finetuning protocol of \citet{amos2023never}, which better reflects the behavior of pretrained transformers. We apply the same protocol to all reproduced baselines for a fair comparison, while reporting the best numbers from prior papers and our training results. All methods share the same architecture except for the attention layer, and we compare against softmax attention, Nyström~\citep{xiong2021nystromformer}, Skyformer~\citep{chen2021skyformer}, Hedgehog~\citep{zhang2024hedgehog}, PolaFormer~\citep{meng2025polaformer}, Flatten Transformer~\citep{han2023flatten}, PolySketchFormer~\citep{kacham2023polysketchformer}, and BASED~\citep{arora2024based}. We set the polynomial degree of \method{} to $p=4$.

Table~\ref{tab:lra_results} summarizes the results. Under this protocol, softmax attention consistently outperforms linear attention variants, aligning with empirical observations from large-scale language modeling. Among linear attention methods, \method{} achieves top-3 performance on every task and obtains the highest average accuracy, improving over the strongest linear attention baseline by $2.28\%$. These results show that \method{} substantially improves linear attention effectiveness while preserving its efficiency advantages.

\begin{table*}[t]
\scriptsize
\caption{Performance comparison across Long Range Arena (LRA) benchmarks. \textbf{Bold numbers} indicate the best performance in linear attention methods, \underline{underlined numbers} indicate the second-best, and \textit{italicized numbers} indicate the third-best.}
\label{tab:lra_results}
\centering
\begin{tabular}{lcccccc}
\toprule
Method & Image & Text & Retrieval & ListOps & Pathfinder & Avg \\
\midrule
Softmax Attn & 64.36 & 82.96 & 83.23 & 48.49 & 70.45 & 69.90 \\
\midrule
Nyström & \textbf{58.66} & 66.19 & \textbf{83.90} & \textit{45.46} & 69.71 & \underline{64.78} \\
Skyformer & \underline{58.54} & 65.25 & 79.55 & \underline{46.77} & 72.73 & \textit{64.57} \\
Hedgehog & 52.68 & \textit{66.59} & \underline{82.24} & 41.03 & \textbf{74.16} & 63.34 \\
PolaFormer & 47.74 & \underline{73.06} & 80.50 & 38.86 & 70.53 & 62.14 \\
Flatten & 51.48 & 66.13 & 78.72 & 38.51 & 70.27 & 61.02 \\
PolySketchFormer & 51.40 & 64.60& 81.10 & 43.00 & 66.20 & 61.26 \\
BASED & 51.13 & 65.73 & 80.76 & 43.55 & \textit{72.98} & 62.83 \\
\textbf{\method{}} & \textit{57.54} & \textbf{73.94} & \textit{82.17} & \textbf{48.29} & \underline{73.84} & \textbf{67.16} \\
\bottomrule
\end{tabular}
\end{table*}

\subsection{Polynomial Kernel Approximation Error}

Furthermore, we directly assess the error of our proposed sketch in approximating polynomial kernels. To that end, we obtain a set of query and key vectors with the following steps: (1) we train a transformer with the exact polynomial kernel ($Attn(\bm{q,k,v})=(\bm{q^\top k)}^p\bm{v}$, $p=2,3$) on the synthetic associative recall (AR) task introduced in~\citet{arora2024based}; (2) we select the first (layer 0), middle (layer 8), and last attention layers (layer 16) of the trained model and sample the corresponding query and key vectors with synthetic AR data. After obtaining the query and key vectors, for each layer, we apply different sketching methods to approximate the polynomial kernel on these sampled vectors.

We compare our sketch against the data-independent sketching approach used in PolySketchFormer~\citep{kacham2023polysketchformer} (PSF), as well as neural network-based sketches (NN). For the neural sketch baseline, the feature maps $\phi_q$ and $\phi_k$ are parameterized as two-layer MLPs and then trained to approximate the exact polynomial kernel. To quantify approximation quality, we report the relative mean absolute error (RMAE), defined as
\begin{equation}
    \mathrm{RMAE}=\mathbf{E}_{\bm{q},\bm{k}}\Bigg|\frac{(\bm{q}^\top\bm{k})^p-\phi_q(\bm{q})^\top\phi_k(\bm{k})}{(\bm{q}^\top\bm{k})^p}\Bigg|.
\end{equation}
The results are summarized in Table~\ref{tab:rmae_sketches}. We observe that our sketch consistently achieves substantially smaller approximation error than PSF. This aligns with the intuition that learnable methods are usually better than data-independent methods. 
Furthermore, our learnable sketch also outperforms the neural network-based feature maps. We attribute this to two main reasons. First, Theorem~\ref{th:optimal} suggests that a linear projection is sufficient to achieve optimal compression for polynomial kernels, making more complex nonlinear parameterizations unnecessary. Second, our structured and decomposed linear formulation is easier to optimize than MLP-based sketches, which may suffer from additional training instability. Overall, these results demonstrate that our proposed sketch provides an effective and accurate approximation of polynomial kernels in practical transformer query and key distributions.


\begin{table}
\vspace{-1em}
\scriptsize
\centering
\caption{Relative mean absolute error of different sketches approximating polynomial kernels with degree $p=2,3$.}
\label{tab:rmae_sketches}
\renewcommand{\arraystretch}{1.2}
\begin{tabular}{c|c|cccc}
\hline
$p$ & Method & Layer 0 & Layer 8 & Layer 15 & Avg \\
\hline
\multirow{3}{*}{$2$}
& PSF       & 0.186 & 0.247 & 0.445 & 0.293 \\
& NN        & 0.035 & 0.070 & 0.058 & 0.054 \\
& \method{} & 0.008 & 0.032 & 0.025 & 0.022 \\
\hline
\multirow{3}{*}{$3$}
& PSF       & 0.740 & 0.738 & 0.745 & 0.741 \\
& NN        & 0.059 & 0.114 & 0.143 & 0.105 \\
& \method{} & 0.018 & 0.040 & 0.058 & 0.038 \\
\hline
\end{tabular}
\vspace{-1em}
\end{table}
\nop{
\subsubsection{Effects of Polynomial Degree}

\begin{table}[t]
\centering
\caption{Performance comparison under different feature dimensions $d$ for Synthetic Association Recall Task}
\begin{tabular}{lcc}
\toprule
Method & $d=64$ & $d=128$ \\
\midrule
BASED & 0.919 & 0.981 \\
\textbf{\method{}} & \textbf{0.943} & \textbf{0.990} \\
\bottomrule
\end{tabular}

\label{tab:dim_results}
\end{table}
}

\subsection{Efficiency Comparison}

We further benchmark the runtime of our proposed linear attention kernel against softmax attention and the BASED kernel. We measure forward-pass latency over 1,000 repetitions with sequence lengths from 1,000 to 10,000 tokens, using the same PyTorch implementation environment. The input dimension is 128. For \method{}, we set the output feature dimension to 256 and $p=2$; for BASED, we compress queries and keys to 16 dimensions, yielding a comparable final feature dimension of $16^2+16+1=273$.

Figure~\ref{fig:efficiency} reports the total forward-pass runtime. As expected, both linear attention methods are much faster than softmax attention at long sequence lengths. Moreover, \method{} is consistently faster than BASED, since it mainly uses matrix multiplications and element-wise products, which are highly optimized on modern accelerators, whereas BASED involves less efficient outer-product-style operations. These results show that \method{} improves approximation quality while retaining strong practical efficiency.

\begin{figure}
    \centering
    \includegraphics[width=0.8\linewidth]{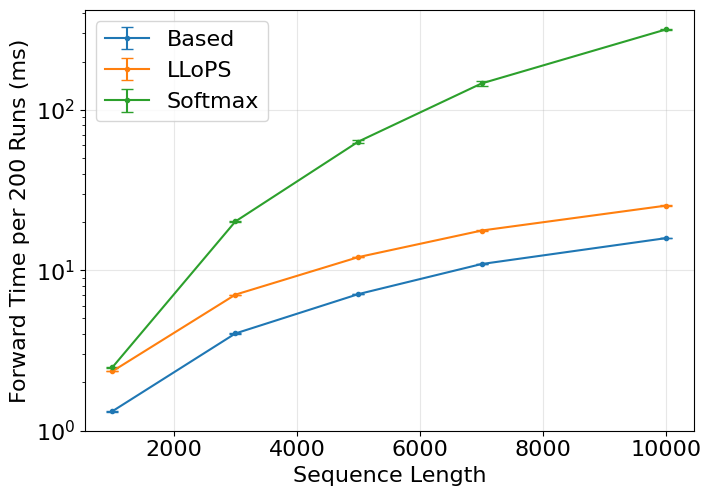}
    \caption{Time cost of different attention kernels at different input length with $p=2$.}
    \label{fig:efficiency}
\end{figure}

\section{Conclusion}

We studied polynomial-kernel-based linear attention through the lens of sketching. We provided a theoretical characterization of the optimal sketch for approximating polynomial kernels in the linear attention setting, revealing that structured and learnable projections can significantly improve approximation quality over previous data-independent feature constructions~\citep{kacham2023polysketchformer}. Building on this insight, we proposed \method{}, a Learnable Decomposed Polynomial Sketch that offers an accurate and efficient approximation to polynomial attention while retaining the favorable computational properties of linear attention.
Empirically, \method{} consistently improves the effectiveness of linear attention Transformers across a diverse set of benchmarks.

\textbf{Limitation}. Since the original pre-training dataset used by BASED is no longer publicly available, we are unable to evaluate \method{} under the same from-scratch training setting. Nevertheless, the results in Table~\ref{tab:nlp_results} show that, when BASED and \method{} are continually trained from the same initialization for the same number of iterations, \method{} consistently achieves higher accuracy. We believe this controlled comparison is sufficient to demonstrate the effectiveness of \method{}.




\bibliography{example_paper}
\bibliographystyle{plainnat}

\appendix
\section{Proofs}

\subsection{Proof of Theorem \ref{th:optimal}}\label{app:proof_optimal}

\begin{proof}
For simplicity, let $\tilde{\bm{Q}}=\psi_p(\bm{Q})$ and $\tilde{\bm{K}}=\psi_p(\bm{K})$. 
Notice that
\begin{equation}
    \tilde{\bm{Q}}\tilde{\bm{K}}^\top=(\bm{QK}^\top)^{\circ p}
\end{equation}

Let $\phi_Q(\bm{Q})=\tilde{\bm{Q}}\bm{W}_1$, $\phi_K(\bm{K})=\tilde{\bm{K}}\bm{W}_2$, where $\bm{W}_1,\bm{W}_2\in\mathbb{R}^{d^p\times m}$. We need to prove that there exist $\bm{W}_1,\bm{W}_2$ such that 
\begin{equation}
    \tilde{\bm{Q}}\bm{W}_1\bm{W}_2^\top \tilde{\bm{K}}^\top=\bm{U_m\Sigma_mV_m}^\top
\end{equation}
where $\bm{S}_m=\bm{U_m\Sigma_mV_m}^\top$ is the truncated SVD decomposition of $\bm{S}=\tilde{\bm{Q}}\tilde{\bm{K}}^\top$.

Notice that the column and row spaces of $\bm{S}_m$ are subspaces of the column and row spaces of $\bm{S}$.

Since $\bm{S}=\tilde{\bm{Q}}\tilde{\bm{K}}^\top$, the columns of $\bm{S}$ are linear combinations of the columns of $\bm{\tilde{Q}}$. Therefore, the range (column space) of $\bm{S}$ is a subspace of the range of $\bm{\tilde{Q}}$:
\begin{equation}
    \text{range}(\bm{S})\subseteq \text{range}(\tilde{\bm{Q}})
\end{equation}
Similarly, the rows of $\bm{S}$ are linear combinations of the rows of $\tilde{\bm{K}}^\top$:
\begin{equation}
    \text{range}(\bm{S}^\top)\subseteq \text{range}(\tilde{\bm{K}})
\end{equation}
As a result, $\text{range}(\bm{U_m})\subseteq \text{range}(\tilde{\bm{Q}})$ and $\text{range}(\bm{V_m})\subseteq \text{range}(\tilde{\bm{K}})$

Therefore, there must exist $\bm{W}_1,\bm{W}_2$ such that $\tilde{\bm{Q}}\bm{W}_1=\bm{U_m\Sigma_m}^{1/2}$ and $\tilde{\bm{K}}\bm{W}_2=\bm{V_m\Sigma_m}^{1/2}$.

So $\tilde{\bm{Q}}\bm{W}_1(\tilde{\bm{K}}\bm{W}_2)^\top=\bm{S}_m$
\end{proof}

\subsection{Proof of Theorem \ref{th:error}} \label{app:proof-error}

\begin{proof}
    \begin{equation}
    \min_{\bm{\Theta}_{1},\ldots,\bm{\Theta}_{p}}\|\bm{W}^*-\bm{\Theta}_{1}\odot\ldots\odot  \bm{\Theta}_{p}\|_F^2=\sum_{i=1}^{m}\min_{\bm{\Theta}_{1,i},\ldots, \bm{\Theta}_{p,i}}
    \|\bm{\Theta}_{1,i}\odot \ldots \odot \bm{\Theta}_{p,i}-\bm{W}^*[:,i]\|_F^2
\end{equation}
where $\bm{\Theta}_{j,i}$ is the $i$-th row of $\bm{\Theta}_j$, $1\le j\le p$, $1\le i\le m$.

By the definition of tensor spectral norm and the best Rank-1 CP approximation, for each $i$,

\begin{equation}
    \min_{\bm{w}_{1,i},\ldots, \bm{w}_{p,i}}
    \|\bm{w}_{1,i}\odot \ldots \odot \bm{w}_{p,i}-\bm{W}^*[:,i]\|_F^2=\|\bm{W}^*[:,i]\|_2^2-\tau_i^2
\end{equation}
\end{proof}

\section{Additional Experiment Results} \label{app:experiment}


\subsection{Synthetic Associative Recall Task}

We compare our method with Based on the synthetic associative reacall task introduced in~\citet{arora2024based}. For both methods, the hidden dimension of the model is set to $d\in\{64,128\}$. The query and key dimensions of BASED is 16. The sketch dimension of \method{} $m=256$ and $p=4$. Table \ref{tab:AR} shows the result. We can see that \method{} achieves higher accuracy than BASED, further proving its effectiveness.

\begin{table}[]
    \centering
    \caption{Comparion of Based and \method{} on synthetic associative recall task. For \method{}, $m=256$, $p=4$. For both methods $d=128$.}
    \label{tab:AR}
    \begin{tabular}{lcc}\toprule
       d  & 128 & 192 \\\midrule
       BASED  & 91.9 & 98.1\\
       \method{} & 93.2 & 99.0\\
       \bottomrule
    \end{tabular}
\end{table}

\subsection{Polynomial Degree Ablation}

Here, we study how the polynomial degree $p$ affects the effectiveness of \method{}. 
We consider the synthetic associative recall task~\citet{arora2024based} and train three 2-layer transformers from scratch using polynomial kernels with degrees $p\in\{2,4,6\}$. 
All models use a hidden dimension of 128, and the output sketch dimension is fixed to 256 across settings. 
The results are summarized in Table~\ref{tab:degree}.

We observe that increasing the degree from $p=2$ to $p=4$ leads to improved accuracy. 
This suggests that moderate higher-order interactions can better distinguish key vectors that align well with the queries from those that do not, thereby enhancing retrieval performance. 
However, when further increasing the degree to $p=6$, accuracy declines. 
We hypothesize that this drop arises because the sketch dimension $m$ remains fixed: as $p$ grows, the polynomial kernel becomes more difficult to approximate accurately, and the resulting increase in approximation error can negatively impact model effectiveness. Another possible reason is that, since the number of parameters increase with $p$, a higher value of $p$ make the model easier to overfit the training data, which would also cause the test accuracy to drop.

Overall, these results highlight an important trade-off: choosing $p$ requires balancing representational power and approximation accuracy, and higher degree is not always better in practice.

\begin{table}[]
    \centering
    \caption{Accuracy of synthetic association recall task under different polynomial degree $p$ with $d=128$.} 
    \label{tab:degree}
    \begin{tabular}{cccc}\toprule
      $p$   & 2 & 4 & 6 \\
      \midrule
       Acc  & 88.7 & 93.2 & 86.6 \\
       \bottomrule
    \end{tabular}
\end{table}

\begin{figure}[htbp]
    \centering
    \begin{subfigure}{0.4\linewidth}
        \centering
        \includegraphics[width=\linewidth]{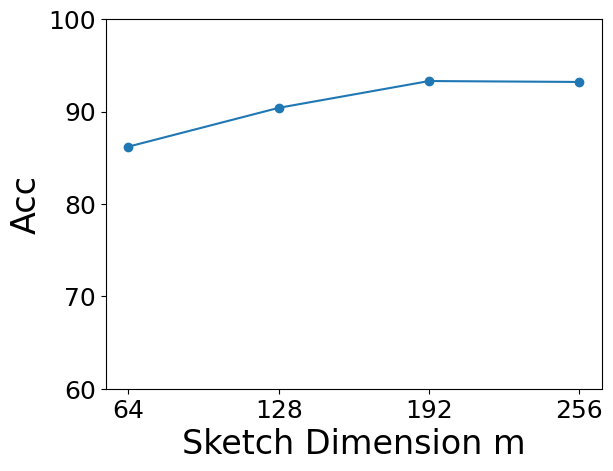}
        \caption{$m$ vs Acc}
        \label{fig:subfig1}
    \end{subfigure}
    \hfill
    \begin{subfigure}{0.4\linewidth}
        \centering
        \includegraphics[width=\linewidth]{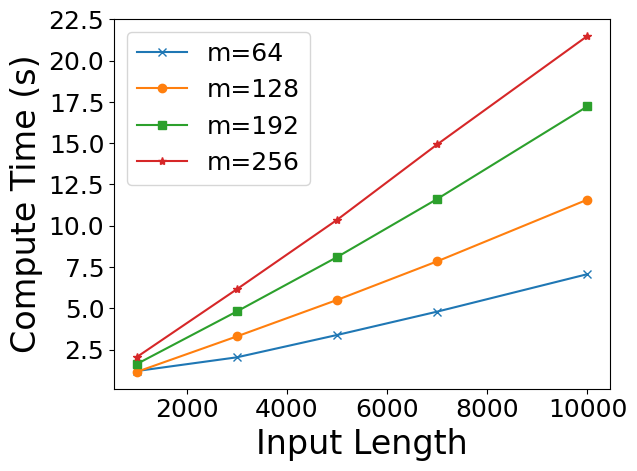}
        \caption{$m$ vs efficency}
        \label{fig:subfig2}
    \end{subfigure}
    \caption{Results of how accuracy and forward time  changes with sketch dimension $m\in\{64,128,192,256\}$ and $d=128$, $p=4$.}
    \label{fig:sketch_dim}
\end{figure}

\subsection{Sketch Dimension Ablation}

Here we demonstrate the effect of sketch dimension $m$ to effectiveness and effciency. We fix the model dimension $d=128$ and vary $m\in\{64,128,192,256\}$. We measure the accuracy for synthetic association recall task~\citep{arora2024based} and the forward computation time. The results show that when $m$ increases, the effective improves but the efficiency drops, which fits the intuition. Nevertheless, we note that when the sketch dimension drops from 256 to 192, the efficiency improves while the accuracy does not drop. This suggests that our method has great potential in approximating polynomial kernels with limited budget of state space size.

\newpage
\section*{NeurIPS Paper Checklist}

The checklist is designed to encourage best practices for responsible machine learning research, addressing issues of reproducibility, transparency, research ethics, and societal impact. Do not remove the checklist: {\bf The papers not including the checklist will be desk rejected.} The checklist should follow the references and follow the (optional) supplemental material.  The checklist does NOT count towards the page
limit. 

Please read the checklist guidelines carefully for information on how to answer these questions. For each question in the checklist:
\begin{itemize}
    \item You should answer \answerYes{}, \answerNo{}, or \answerNA{}.
    \item \answerNA{} means either that the question is Not Applicable for that particular paper or the relevant information is Not Available.
    \item Please provide a short (1--2 sentence) justification right after your answer (even for \answerNA). 
\end{itemize}

{\bf The checklist answers are an integral part of your paper submission.} They are visible to the reviewers, area chairs, senior area chairs, and ethics reviewers. You will also be asked to include it (after eventual revisions) with the final version of your paper, and its final version will be published with the paper.

The reviewers of your paper will be asked to use the checklist as one of the factors in their evaluation. While \answerYes{} is generally preferable to \answerNo{}, it is perfectly acceptable to answer \answerNo{} provided a proper justification is given (e.g., error bars are not reported because it would be too computationally expensive'' or ``we were unable to find the license for the dataset we used''). In general, answering \answerNo{} or \answerNA{} is not grounds for rejection. While the questions are phrased in a binary way, we acknowledge that the true answer is often more nuanced, so please just use your best judgment and write a justification to elaborate. All supporting evidence can appear either in the main paper or the supplemental material, provided in appendix. If you answer \answerYes{} to a question, in the justification please point to the section(s) where related material for the question can be found.

IMPORTANT, please:
\begin{itemize}
    \item {\bf Delete this instruction block, but keep the section heading ``NeurIPS Paper Checklist"},
    \item  {\bf Keep the checklist subsection headings, questions/answers and guidelines below.}
    \item {\bf Do not modify the questions and only use the provided macros for your answers}.
\end{itemize}


\begin{enumerate}

\item {\bf Claims}
    \item[] Question: Do the main claims made in the abstract and introduction accurately reflect the paper's contributions and scope?
    \item[] Answer: \answerYes{} 
    \item[] Justification: The main claims made in the abstract and introduction accurately reflect the paper's contributions and scope.

\item {\bf Limitations}
    \item[] Question: Does the paper discuss the limitations of the work performed by the authors?
    \item[] Answer: \answerYes{} 
    \item[] Justification: The limitation is discussed in conclusion.

\item {\bf Theory assumptions and proofs}
    \item[] Question: For each theoretical result, does the paper provide the full set of assumptions and a complete (and correct) proof?
    \item[] Answer: \answerYes{} 
    \item[] Justification: The proofs are provided in the main paper and appendix.

    \item {\bf Experimental result reproducibility}
    \item[] Question: Does the paper fully disclose all the information needed to reproduce the main experimental results of the paper to the extent that it affects the main claims and/or conclusions of the paper (regardless of whether the code and data are provided or not)?
    \item[] Answer: \answerYes{} 
    \item[] Justification: All information is disclosed.

\item {\bf Open access to data and code}
    \item[] Question: Does the paper provide open access to the data and code, with sufficient instructions to faithfully reproduce the main experimental results, as described in supplemental material?
    \item[] Answer: \answerNo{} 
    \item[] Justification: The code will be released upon acceptance.

\item {\bf Experimental setting/details}
    \item[] Question: Does the paper specify all the training and test details (e.g., data splits, hyperparameters, how they were chosen, type of optimizer) necessary to understand the results?
    \item[] Answer: \answerYes{} 
    \item[] Justification: All information is disclosed

\item {\bf Experiment statistical significance}
    \item[] Question: Does the paper report error bars suitably and correctly defined or other appropriate information about the statistical significance of the experiments?
    \item[] Answer: \answerYes{} 
    \item[] Justification: We include error bars in Figure \ref{fig:efficiency}. For other experiments, due to the high cost of pre-training, we cannot use different random seeds to train it multiple times.

\item {\bf Experiments compute resources}
    \item[] Question: For each experiment, does the paper provide sufficient information on the computer resources (type of compute workers, memory, time of execution) needed to reproduce the experiments?
    \item[] Answer: \answerYes{} 
    \item[] Justification: Information is disclosed
    \item[] Guidelines:
    
\item {\bf Code of ethics}
    \item[] Question: Does the research conducted in the paper conform, in every respect, with the NeurIPS Code of Ethics \url{https://neurips.cc/public/EthicsGuidelines}?
    \item[] Answer: \answerYes{} 
    \item[] Justification: This paper conform with the NeurIPS code of ethics.

\item {\bf Broader impacts}
    \item[] Question: Does the paper discuss both potential positive societal impacts and negative societal impacts of the work performed?
    \item[] Answer: \answerYes{} 
    \item[] Justification: The impact of the paper is discussed in the main paper.
    
\item {\bf Safeguards}
    \item[] Question: Does the paper describe safeguards that have been put in place for responsible release of data or models that have a high risk for misuse (e.g., pre-trained language models, image generators, or scraped datasets)?
    \item[] Answer: \answerNA{} 
    \item[] Justification: The paper has no such risk.

\item {\bf Licenses for existing assets}
    \item[] Question: Are the creators or original owners of assets (e.g., code, data, models), used in the paper, properly credited and are the license and terms of use explicitly mentioned and properly respected?
    \item[] Answer: \answerYes{} 
    \item[] Justification: They are credited.

\item {\bf New assets}
    \item[] Question: Are new assets introduced in the paper well documented and is the documentation provided alongside the assets?
    \item[] Answer: \answerNA{} 
    \item[] Justification: the paper does not release new assets.

\item {\bf Crowdsourcing and research with human subjects}
    \item[] Question: For crowdsourcing experiments and research with human subjects, does the paper include the full text of instructions given to participants and screenshots, if applicable, as well as details about compensation (if any)? 
    \item[] Answer: \answerNA{} 
    \item[] Justification: the paper does not involve crowdsourcing nor research with human subjects.

\item {\bf Institutional review board (IRB) approvals or equivalent for research with human subjects}
    \item[] Question: Does the paper describe potential risks incurred by study participants, whether such risks were disclosed to the subjects, and whether Institutional Review Board (IRB) approvals (or an equivalent approval/review based on the requirements of your country or institution) were obtained?
    \item[] Answer: \answerNA{} 
    \item[] Justification: the paper does not involve crowdsourcing nor research with human subjects.

\item {\bf Declaration of LLM usage}
    \item[] Question: Does the paper describe the usage of LLMs if it is an important, original, or non-standard component of the core methods in this research? Note that if the LLM is used only for writing, editing, or formatting purposes and does \emph{not} impact the core methodology, scientific rigor, or originality of the research, declaration is not required.
    \item[] Answer: \answerNA{} 
    \item[] Justification: the core method development in this research does not involve LLMs as any important, original, or non-standard components.

\end{enumerate}

\end{document}